\documentclass[11pt]{article} % For LaTeX2e
\usepackage[preprint]{tmlr}

\usepackage{amsmath,amsfonts,bm}

\def\eqref#1{equation~\ref{#1}}
\def\1{\bm{1}}

\def\ry{{\textnormal{y}}}

\def\rvx{{\mathbf{x}}}

\def\rvz{{\mathbf{z}}}

\def\vx{{\bm{x}}}

\DeclareMathAlphabet{\mathsfit}{\encodingdefault}{\sfdefault}{m}{sl}
\SetMathAlphabet{\mathsfit}{bold}{\encodingdefault}{\sfdefault}{bx}{n}

\pdfoutput=1
\usepackage{newtxtext}
\usepackage{url}
\usepackage[utf8]{inputenc} % allow utf-8 input
\usepackage[T1]{fontenc}    % use 8-bit T1 fonts

\usepackage{booktabs}       % professional-quality tables
\usepackage{amsfonts}       % blackboard math symbols
\usepackage{nicefrac}       % compact symbols for 1/2, etc.
\usepackage{microtype}      % microtypography

\usepackage{graphicx}

\usepackage{url}
\usepackage[symbol]{footmisc}

\usepackage{todonotes}
\usepackage{algorithm,algorithmic}
\usepackage{amsmath}
\usepackage{amssymb}
\usepackage{caption}
\usepackage{subcaption}
\usepackage{amsthm}
\usepackage{bbm}
\usepackage{array,multirow}
\usepackage{wrapfig}
\usepackage[normalem]{ulem}
\usepackage{makecell}
\usepackage{pifont}
\usepackage{wrapfig}
\usepackage[export]{adjustbox}
\usepackage[english]{babel}
\addto{\captionsenglish}{%
    
}
\usepackage{mdframed}   % for framing

\definecolor{mydarkblue}{rgb}{0,0.08,0.45}
\usepackage[colorlinks,citecolor=mydarkblue,urlcolor=mydarkblue,linkcolor=mydarkblue,filecolor=mydarkblue]{hyperref}

\theoremstyle{plain}
\newtheorem{theorem}{Theorem}[section]
\newtheorem{proposition}[theorem]{Proposition}
\newtheorem{lemma}[theorem]{Lemma}
\newtheorem{corollary}[theorem]{Corollary}
\theoremstyle{definition}
\newtheorem{definition}[theorem]{Definition}

\theoremstyle{remark}

\title{On Calibration in Multi-Distribution Learning}

\author{\name Rajeev Verma \email r.verma@uva.nl \\
      \addr UvA-Bosch Delta Lab\\
      University of Amsterdam
      \AND
      \name Volker Fischer \email  volker.fischer@de.bosch.com\\
      \addr Bosch Center for Artificial Intelligence \\
      \addr Renningen, Germany
      \AND
      \name Eric Nalisnick \email nalisnick@jhu.edu\\
      \addr Department of Computer Science \\
      Johns Hopkins University}

\def\month{MM}  % Insert correct month for camera-ready version
\def\year{YYYY} % Insert correct year for camera-ready version
\def\openreview{\url{https://openreview.net/forum?id=XXXX}} % Insert correct link to OpenReview for camera-ready version

\begin{document}

\maketitle

\begin{abstract}
Modern challenges of robustness, fairness, and decision-making in machine learning have led to the formulation of \textit{multi-distribution learning} (MDL) frameworks in which a predictor is optimized across multiple distributions. We study the calibration properties of MDL to better understand how the predictor performs uniformly across the multiple distributions. Through classical results on decomposing proper scoring losses, we first derive the Bayes optimal rule for MDL, demonstrating that it maximizes the generalized entropy of the associated loss function. Our analysis reveals that while this approach ensures minimal worst-case loss, it can lead to non-uniform calibration errors across the multiple distributions and there is an inherent calibration-refinement trade-off, even at Bayes optimality. Our results highlight a critical limitation: despite the promise of MDL, one must use caution when designing predictors tailored to multiple distributions so as to minimize disparity.
\end{abstract}

\section{Introduction}
\looseness=-1The ever increasing applicability of machine learning systems in safety-critical applications requires robustness, fairness, and reliable decision-making. Contrary to traditional machine learning, which is operationalized by assuming samples originate from a single fixed distribution, \textit{multi-distribution learning} (MDL) aims to learn a predictor across a set of distributions. Moreover, the predictor should perform well for \textit{any distribution} in the set.  MDL has the potential to provide a unified framework for tackling the aforementioned challenges. For example, to improve robustness, MDL can be applied by specifying the set as distributions within a small perturbation of the training distribution.  This setup is also known as \textit{distributionally robust optimization}, and thus MDL can be seen as a generalization. To turn to fairness, equitable decisions must be made across a collection of heterogeneous populations defined by protected group-membership.  This setting can be considered a form of MDL by taking the heterogeneous populations to be the set of interest. A related problem is that of decision-making where the learner's loss is aggregated by averages (or expectations). However, taking averages can cause disparate performance across populations \citep{blum2021communication} and also cause a fair model to become unfair \citep{hashimoto2018fairness}. %More fundamentally, the goal is to go beyond the \textit{de-facto} independent and identically distributed (i.i.d.) set-up of traditional machine learning.

\looseness=-1A unified study of MDL in learning theory \citep{haghtalab2022ondemand, blum2017collaborative, mohri2019agnostic, zhang2024optimal} extends the celebrated agnostic learning framework \citep{valiant1984theory, kearns1992toward} and has primarily focused on the sample complexity of learning across the multiple distributions.  For instance, \citet{haghtalab2022ondemand} propose adaptive sampling to achieve data-efficient MDL algorithms. However, there are no studies on the calibration \citep{Dawid1982TheWB} of predictive models under the MDL framework. Calibration properties of predictors under the traditional single-distribution framework are well-documented \citep{Degroot1983TheCA, brocker2009reliability}, and have long guided the practice of machine learning. Besides the general discriminativeness (refinement) of predictors, calibration has important consequences in safety-critical applications, especially to drive equitable decision-making \citep{grunwald2018safe, dwork2021outcome, rothblum2021multi, zhao2021calibrating}. For example, perfect calibration guarantees that a healthcare facility can \textit{reliably} use the same threshold for everyone on the predicted risk of a disease to recommend further expensive tests. Since the motivations of MDL of robustness and fairness aligns with the usual motivations of calibration, it is imperative to study how well these two properties align theoretically.

\looseness=-1In this work, we take a foundational approach and study the calibration of a predictor under the MDL framework, independent of specific learning algorithms. One of the challenges in studying calibration for MDL is that the MDL learning algorithms are usually randomized complicating their tractable study. Thus, we also give an attainable lower bound for MDL and study the calibration at this lower bound.  Our work reveals the existence of non-uniform calibration disparities across different distributions in MDL---a problem that resembles calibration disparity across different (sub)groups in the single distribution literature \citep{pmlr-v80-hebert-johnson18a}. However, while it can be eliminated theoretically and algorithmically for the  single-distribution case, calibration disparity in MDL is \textit{fundamental} and there exists a \textit{calibration-refinement trade-off}. Thus, our analysis shows that the promise of MDL to aid robustness and equitable decision-making is tenuous.

\paragraph{Summary of Results} 1. Proposition \ref{thm:attainable-lower-bound} states the attainable lower bound on the error of a predictor in MDL. 2. Proposition \ref{prop:calibration-error-bound-in-terms-of-entropy-difference} and the next results establish calibration error bound for any distribution in MDL and characterize the calibration disparity in terms of the performance disparity for a predictor across multiple distributions---arguing that if a predictor leads to low error for some distribution, it compensates for it by increasing the calibration error. Proposition \ref{prop:multi-distribution-learning-calibration-and-decision-theoretic-property} states the decision-theoretic consequence of calibration in MDL, and argue that decision-makers need to be careful to fully realize the potential of MDL. 4. Finally, Section \ref{sec:consequences} states the consequences of our findings to the use case of distributional robust optimization and fairness, two popular applications of MDL. 

% We first derive the Bayes optimal rule for MDL, and argue that if this Bayes optimality is achieved, it can naturally lead to calibration disparities across different distributions in MDL. 2. We then characterize this calibration disparity in terms of the performance disparity for a predictor across multiple distributions---arguing that if a predictor leads to low error for some distribution, it compensates for it by increasing the calibration error. We further show the decision-theoretic consequence of calibration in MDL, and argue that decision-makers need to be careful to fully realise the potential of MDL. 3. We show the consequence of our findings for distributional robust optimization arguing for one potential design choice to achieve robust as well as equitable decision-making. 4. We further connect our findings to general class of risk measures, and characterizing their calibration behavior for the first time.

\section{Background}
In this section, we motivate MDL framework and give background for calibration and refinement and its connection with risk minimization.

\paragraph{Notation and Single Distribution Learning}
Traditionally, practical machine learning starts with an access to some finite dataset $D = \{\left(\vx_{i}, y_{i}\right)\}_{i=1}^{N}$, where $\left(\vx, y\right) \in \mathcal{X} \times \mathcal{Y}$. The goal, then, is to design a predictor $h_{\theta}: \mathcal{X} \rightarrow \Delta^{|\mathcal{Y}|}$, $\Delta^{|\mathcal{Y}|}$ denotes a simplex over the label space $\mathcal{Y}$ where we assume $|\mathcal{Y}|$ is finite. Here, this predictor is parameterized by $\theta$, and is a member of some hypothesis class $\mathcal{H}= \{h_{\theta}: \theta \in \Lambda\}$. However, moving forward we will suppress the dependence of the hypothesis $h$ on $\theta$. Given some pre-specified loss $\ell: \mathcal{Y} \times \Delta^{|\mathcal{Y}|} \rightarrow [0,M]$, one aims to find $h^{*}$ as a solution to the following optimization problem $$h^{*} \in \arg\min_{\mathcal{H}}\frac{1}{N}\sum_{i=1}^{N}\left[\ell\left(y_i, h\left(\vx_{i}\right)\right)\right],$$
To elaborate, we assume the loss function $\ell$ is bounded. If the loss function is not bounded from above, e.g. the log-loss $-\log h\left(\vx\right)$, we assume it can be bounded by re-defining it as $\min\left\{M, -\log h\left(\vx\right)\right\}$, for example. The above exercise makes sense if one is to assume some base measure $P$ on $\mathcal{X} \times \mathcal{Y}$, as then the law of large numbers suggests that the goodness of the predictor $h^{*}$ as measured by the loss $\ell$ is controlled in terms of the expectation over the population $P$, i.e. the expected population loss, also referred to as the risk denoted as $\mathcal{R}_{P} = \mathbb{E}_{P}\left[\ell\left(\ry, h\left(\rvx\right)\right)\right]$, is controlled. We use $\rvx$ and $\ry$ to denote random variables. 

\subsection{Calibration and Refinement}\label{subsec:calibration-and-refinement}
While the learning procedure stated above controls the risk, decision makers also care about the \textit{confidence} of a predictor. The output of the hypothesis $h\left(\vx\right) \in \Delta^{|\mathcal{Y}|}$ is considered as the forecast over $\mathcal{Y}$, and to reliably use it as a notion of confidence, the decision makers aim to certify the calibration property of $h$ as defined below:

\begin{definition}\label{def:canonical-calibration}
    \textit{(Canonical Calibration)}. 
    Given $d$ some divergence measure, e.g. squared error, a  confidence predictor $h: \mathcal{X} \rightarrow \Delta^{|\mathcal{Y}|}$ is said to be (perfectly) canonically calibrated if the following holds true:
    \begingroup
    \setlength{\abovedisplayskip}{5pt}
    \setlength{\belowdisplayskip}{5pt}
    \begin{equation}\label{eq:canonical_calibration}
    \mathbb{E}_{\left(\rvx, \ry\right)\sim P}\left[d\left(P\left[ \ \ry \ \vert \ h\left(\rvx\right) \ \right], h\left(\rvx\right)\right)\right] = 0. 
    \end{equation}
    \endgroup
\end{definition}
Canonical calibration asserts that on average the confidence predictor means what it says, i.e. $P\left[\ry \ \vert \ h\left(\rvx\right)\right] = h\left(\rvx\right).$ When $\mathcal{Y} = \{0,1\}$, this translates that if the predictor outputs that the confidence in some event is $\alpha$, among all the samples that have the same confidence $\alpha$, the event will occur $\alpha$ times on average.  In addition to providing interpretability and trustworthiness for the predictor, calibration bridges the inherent institutional separation between machine learning experts and downstream decision-makers, as described below.

\paragraph{Decision-Theoretic Property of Canonical Calibration} 
% \textcolor{red}{Seems like some of this paragraph can be cut (if we need space), especially since you give the formal result in Prop 2.2.} 
% As hinted before,  
A machine learning model designed for some loss function $\ell$ (e.g. log-loss), once deployed, is generally consumed by decision-makers with arbitrary cost considerations in mind. For example, a medical facility using a risk prediction system to aid diagnosis could decide either to treat the patient or not to treat the patient depending on some cost function, even when the predicted diagnosis is positive. In such cases, a decision-maker wants to reliably
use the said predictor for decision-making with arbitrary cost considerations.  Canonical calibration provides that reliability per the framework of expected utility \citep{grunwald2018safe, zhao2021calibrating, kleinberg2023u, derr2024four}. For completeness, we state the result from \citet{let4allcalibrationblog}:
\begin{proposition}\label{prop:calibration-and-decision-making-regular-calibration}
    \textit{(Calibration and decision-making)}. \citep{let4allcalibrationblog}. Given $\mathcal{X} \times \mathcal{Y}$ with a distribution $P$ on it, a (finite) action space $\mathcal{A}$ and an arbitrary cost function $c: \mathcal{Y} \times \mathcal{A} \rightarrow \mathbb{R}_{+}$,  the decision-rule  $\delta: h\left(\rvx\right) \mapsto \arg\min_{a \in \mathcal{A}}\mathbb{E}_{\ry \sim h\left(\rvx\right)}\left[c\left(\ry, a\right)\right]$ based on the canonically calibrated predictor $h: \mathcal{X} \rightarrow \Delta^{|\mathcal{Y}|}$ is optimal on average across $P$.
\end{proposition}
% \textcolor{red}{the above statement simply gives direction on how to use to the calibrated predictor; it does not tell why do we need calibration?} 
Intuitively, the above statement says that a decision maker can treat the predicted probabilities as a proxy for the target probability to base their decisions optimally on average across the population, and this holds true for any arbitrary cost function. In case of a binary action space, the decision-rule becomes a simple thresholding rule on the predicted forecast based on the considered cost function (Refer Section \ref{sec:calibration-thresholding-rule}). Thus, there is a huge incentive to prioritize calibrated predictors. However, calibration by itself, is a weak condition, and can be trivially satisfied, for example, by the marginal constant predictor $h\left(\rvx\right) = \mathbb{E}\left[\ry\right]$, and there are infinitely many calibrated predictors \citep{pmlr-v89-vaicenavicius19a}. Thus, the confidence predictor is also verified for its \textit{refinement} as defined below:

\begin{definition}
    \textit{(Refinement error)}. 
    % For $\mathcal{X} \times \mathcal{Y}$ with a joint distribution $P$ on it, 
    Let $H$ some notion of information (e.g.~entropy).  The refinement error of a confidence predictor $h: \mathcal{X} \rightarrow \Delta^{|\mathcal{Y}|}$ is defined as the average information content: $
        \mathbb{E}_{\left(\rvx, \ry\right)\sim P}\left[H\left(P\left[\ \ry \ \vert \ h\left(\rvx\right)\right]\right)\right]$.
\end{definition}
Intuitively, refinement means that, on average, the confidence predictor is useful for predicting $\ry$, and low refinement error signals the \textit{discriminativeness} of the confidence predictor. Hence, a calibrated predictor with low refinement error is usually the goal.

\subsection{Risk Minimization and Calibration}
A machine learning practitioner expects to obtain the confidence predictor with both the low calibration error and the low refinement error natively as a result of the expected risk minimization framework. This is usually a result of employing proper scoring loss functions \citep{properscoringrulesGneiting} in risk minimization. Proper scoring losses follow from the property elicitation literature to evaluate probability forecasters. 

\paragraph{Proper Scoring Losses} Given the sample $\vx$, a proper scoring loss $\ell$ aims to evaluate the quality of the forecast $h\left(\vx\right) \in \Delta^{|\mathcal{Y}|}$ against the target forecast $\eta\left(\vx\right) := P\left(\ry \ \vert \rvx=\vx\right)$ as $L\left(\eta\left(\vx\right), h\left(\vx\right)\right) := \mathbb{E}_{\ry \sim \eta\left(\vx\right)}\left[\ell\left(\ry, h\left(\vx\right)\right)\right]$, and the notion of propriety means that $L\left(\eta\left(\vx\right), h\left(\vx\right)\right) - L\left(\eta\left(\vx\right), \eta\left(\vx\right)\right) \geq 0, \ \forall \eta\left(\vx\right), h\left(\vx\right) \in \Delta^{|\mathcal{Y}|}$. When this inequality is strict, the loss is called strictly proper. In words, minimizing the expectation of the proper scoring loss recovers the target forecast, i.e. if one is to minimize $L\left(\eta\left(\vx\right), h\left(\vx\right)\right)$ for $h\left(\vx\right)$, the minimizer is attained when $h\left(\vx\right) = \eta\left(\vx\right)$, or $L\left(\eta\left(\vx\right), \eta\left(\vx\right)\right) = \inf_{h\left(\vx\right)}L\left(\eta\left(\vx\right), h\left(\vx\right)\right)$. The latter quantity is named as the Bayes risk associated with the target forecast $\eta\left(\vx\right)$, and is named as the \textit{generalized entropy} \citep{grunwald2004game} of $\eta\left(\vx\right)$ associated with the loss function $\ell$, denoted as $H_{\ell}\left(\eta\left(\vx\right)\right)$. Since the generalized entropy is a point-wise infimum of linear functionals, if follows that $H_{\ell}\left(\eta\right)$ is concave in the forecasts $\eta \in \Delta^{|\mathcal{Y}|}$. In turn, the generalized entropy function gives an interesting geometric characterization of the proper scoring losses, which we formally state below:

\begin{definition}
    \textit{(Characterizing proper scoring loss via the generalized entropy function)}. \citep{ovcharov2018proper}. A scoring loss $\ell : \mathcal{Y} \times \Delta^{|\mathcal{Y}|} \rightarrow \mathbb{R}_{+}$ is called (strictly) proper iff there exists a (strictly) concave function $H: \Delta^{|\mathcal{Y}|} \rightarrow \mathbb{R}$ and a sub-gradient $\Delta H : \Delta^{|\mathcal{Y}|} \rightarrow \mathcal{L}\left(\Delta^{|\mathcal{Y}|}\right)$ (the set of linear functionals or hyperplanes over the span $\Delta^{\mathcal{Y}}$) of $H$ such that $$\ell\left(y, h\left(\vx\right)\right) = H\left(h\left(\vx\right)\right) + \Delta H\left(h\left(\vx\right)\right)\cdot \left(\delta^{y} - h\left(\vx\right)\right), \ \forall h\left(\vx\right) \in \Delta^{|\mathcal{Y}|}.$$
\end{definition}

Intuitively, the risk $\mathbb{E}_{\ry \sim \eta\left(\vx\right)}\left[\ell\left(\ry, h\left(\vx\right)\right)\right] = H\left(h\left(x\right)\right) + \Delta H\left(h\left(x\right)\right)\cdot \left(\eta\left(\vx\right) - h\left(\vx\right)\right)$ is a hyperplane at the prediction $h\left(\vx\right)$ evaluated at the target forecast value $\eta\left(\vx\right)$. Due to the concavity of the entropy function $H$, the minimization of risk of $h\left(\vx\right)$ happens when $h\left(\vx\right)$ coincides with $\eta\left(\vx\right)$, and it attains $H\left(\eta\left(\vx\right)\right)$--- the minimal risk possible for the prediction $h\left(\vx\right)$. This geometric interpretation also naturally leads to defining a Bregman divergence \citep{bregman1967relaxation} associated with a proper scoring loss. Traditionally, Bregman divergences are associated with a convex functions and defines the distance between two points $a$ and $b$ as the difference between the function evaluated at $b$ and the value of the tangent (hyperplane) at $a$ evaluated at $b$. However, one can use a similar notion for a concave entropy function to define the associated Bregman divergence $d_{\ell}: \Delta^{|\mathcal{Y}|} \times \Delta^{|\mathcal{Y}|} \rightarrow \mathbb{R}_{+}$ for a loss $\ell$ as $d_{\ell}\left(\eta\left(\vx\right), h\left(\vx\right)\right) = H\left(h\left(\vx\right)\right) + \Delta H\left(h\left(\vx\right)\right)\cdot \left(\eta\left(\vx\right) - h\left(\vx\right)\right) - H\left(\eta\left(\vx\right)\right) = L\left(\eta\left(\vx\right), h\left(\vx\right)\right) - H\left(\eta\left(\vx\right)\right).$ For instance, for log-loss $\ell\left(y, h\left(\vx\right)\right) = -\log h\left(\vx\right)\left(y\right)$, $\text{KL}\left(\eta\left(\vx\right) \ || \ h\left(\vx\right)\right)$ is the associated Bregman divergence. 

Given the above connections between entropy functions, proper scoring rules, and their associated Bregman divergences, risk minimization of a proper scoring loss decomposes into the calibration error and the refinement error, stated formally below:

\begin{lemma}
    \textit{(Decomposition of proper scoring risk into calibration and refinement error)}. \citep{brocker2009reliability}. Given a space $\mathcal{X} \times \mathcal{Y}$ with a distribution $P$ specified on it, a confidence predictor $h : \mathcal{X} \rightarrow \Delta^{|\mathcal{Y}|}$ whose risk in expectation over $P$ is evaluated by a proper scoring loss $\ell: \mathcal{Y} \times \Delta^{|\mathcal{Y}|} \rightarrow [0, M]$, the said risk decomposes into the calibration error term and the refinement term, as below:
    \begin{align}
        \begin{split}
           \mathbb{E}\left[\ell\left(\ry, h\left(\rvx\right)\right)\right] = \underbrace{\mathbb{E}\left[d_{\ell}\left(P\left(\ry \ \vert \ h\left(\rvx\right)\right), h\left(\rvx\right)\right)\right]}_{\text{calibration error}}  \ +  \ \underbrace{\mathbb{E}\left[H_{\ell}\left(P\left(\ry \ \vert \ h\left(\rvx\right)\right)\right)\right]}_{\text{refinement}},
        \end{split}
    \end{align}
    where $d_{\ell}$ and $H_{\ell}$ are respectively the Bregman divergence and the generalized entropy function associated with the loss $\ell$.
\end{lemma}
Thus, if one is to minimize a (strictly) proper scoring loss function, one expects the predictor to have low calibration error and low refinement error. Since $d_{\ell} \geq 0$ and $H\left(P\left(\ry \ \vert \ h\left(\rvx\right)\right)\right) \geq H\left(P\left(\ry \ \vert \ \rvx\right)\right)$,  minimizing the right hand side recovers the true Bayes predictor $h^{*}\left(\rvx\right) = P\left(\ry \vert \ \rvx\right)$, that has zero calibration error. 

\section{Multi-Distribution Learning (MDL)}\label{sec:MDL}
%Despite the promise of low expected loss across the population that learning over the single distribution provides, several limitations motivate the need for reformulation. In particular, the core setup of the base measure $P$ and the control in terms of expectations require revising.

Although single-distribution learning has provided much success to date, applications of modern machine learning challenge the assumptions of having a single base measure $P$ and using expectations for control.  For instance, distribution shift often degrades our predictive models' performance in the wild, and this can be seen as originating from over-reliance on the single measure $P$.  Moreover, using averages to summarize losses presupposes that the decision maker has the disposition of a risk neutral agent, while safety-critical applications might demand a safer, risk-averse attitudes.  In order to fundamentally combat these challenges, adjustments have been proposed to the traditional practice of machine learning that all take the form of multi-objective optimization or MDL over the set of distributions \citep{haghtalab2022ondemand, zhang2024optimal} or learning under \textit{hallucinated ambiguity} \citep{riskmeasuresandupperprobabilities}.  We summarize these approaches below.

% However, it is worth realizing that the choice of a base measure is merely an assumption, and the goodness of the predictor as measured by the expectation under this base measure is a design choice \citep{riskmeasuresandupperprobabilities}. This realization sheds new light on the failure modes of the machine learning systems. 
%For instance, over-reliance on the base measure $P$ explains the erratic behavior of the machine learning models under arbitrary shifts (distribution shift or local perturbations). Safety assurances based on averages hint at the disparity machine learning systems induce for protected or marginalized (sub)populations. The design choice of averages further presupposes that the decision maker has the disposition of a risk neutral agent, while the safety-critical applications might demand a safer risk-averse attitude. 

%In order to combat these challenges fundamentally, adjustments have been proposed to the traditional practice of machine learning that all take the form of multi-objective optimization or MDL over the set of distributions \citep{haghtalab2022ondemand, zhang2024optimalmultidistributionlearning} or learning under \textit{hallucinated ambiguity} \citep{riskmeasuresandupperprobabilities}.

\paragraph{Setting and Goal of MDL} Generally in MDL, one assumes a set $\mathcal{Q}$ of distributions over $\mathcal{X} \times \mathcal{Y}$, and the goal is to design a universal predictor that works uniformly well across all the distributions $Q \in \mathcal{Q}$. The set $\mathcal{Q}$ can take different forms depending on the applications scenario, but generally we assume the set to be compact. In the single distribution case, the goals of the learning problem are well-defined in that one aims to recover (or approach) the Bayes optimal predictor $h^{*}\left(\rvx\right) = P\left(\ry \ \vert 
 \ \rvx\right)$, and the Bayes risk for loss $\ell$, $H_{\ell}\left(P\left(\ry \ \vert \ \rvx\right)\right)$, serving as a benchmark for performance. In MDL, one can extend this notion of Bayes optimality to define the generalized Bayes rule, formally written below:

\begin{definition}
    \textit{(Generalized Bayes rule)}. \citep{frohlich2024data}. Given the set of distributions $\mathcal{Q}$ over $\mathcal{X} \times \mathcal{Y}$, the generalized Bayes rule is the conditional distribution set: $B^{*}_{\mathcal{Q}}\left(\vx\right) =\{Q\left(\ry \ \vert \ \rvx=\vx\right) \ : \ Q \in \mathcal{Q}\}$. The generalized Bayes score for the loss function $\ell$ is $\{H_{\ell}\left(Q\left(\ry \ \vert \ \rvx\right)\right) \ : \ Q \in \mathcal{Q}\}$. 
    % \textcolor{red}{We take that $H_{\ell}(Q(\ry|\rvx))$ is finite. Why?---$H_{\ell}(Q(\ry|\rvx)) = \inf_{h \in \mathcal{H}}\mathbb{E}_{\ry \sim Q(\ry|\rvx)}\left[\ell(\ry, h)\right]$. $\mathcal{Y}$ is finite. If $\ell$ is also bounded from above, then finiteness holds obviously. But even if $\ell$ is not bounded from above: if $\mathcal{H}$ is unrestricted, then $Q(\ry|\rvx)$ belongs in $\mathcal{H}$, then $\mathbb{E}_{\ry \sim Q(\ry|\rvx)}\ell(\ry, Q(\ry|\rvx))$ being infinity does not sit right---it simply raises concerns about $\ell$ being a ``good'' loss function. We can simply focus on ``good'' loss functions where such pathologies do not occur.  }
\end{definition}

 However, the desired notion of universality in MDL deems the generalized Bayes rule impractical. Universality implies that  a decision maker can observe a sample $\rvx \times \ry$ from any distribution $Q \in \mathcal{Q}$, and the goal of the learning problem is to design a (single) predictor that does not put the decision maker at too much risk. Several notions of universality across $\mathcal{Q}$ have been studied in the literature ranging from controlling averages \citep{blanchard2011generalizing,blanchard2021domain} to worst-case aggregation \citep{arjovsky2020out, Ben-Tal-RobustOptimization, Sagawa*2020Distributionally}, and the interpolations between them \citep{eastwood2022probable}. Thus, the goal of MDL is intimately connected to the notion of universality employed. This connection introduces ambiguity, as different notions of universality lead to varying objectives. Our first technical result is to state the attainable lower bound on the error of any universal predictor. This result provides a foundational benchmark for evaluating predictors in MDL, abstracting away the need for explicit assumptions about universality.

 The result is best motivated considering the log-loss, and recalling the seminal Shannon's source coding theorem which gives the lower limit on the compression ability for any source in terms of the expected Shannon's entropy. Considering $\ell$ to be log-loss, and viewing learning as compression, the same result gives the lower limit on the risk of any predictor involving the distribution $Q$ as $H\left(Q\left(\ry \ \vert \ \rvx\right)\right)$ where $H$ is the Shannon's entropy. Recalling from above, this is the same as the generalized entropy associated with the log-loss, i.e. $\inf_{h \in \mathcal{H}}\mathbb{E}_{Q}\left[\ell\left(\ry, h\left(\rvx\right)\right)\right]$. For MDL with some other loss function $\ell$ over the set of distributions $\mathcal{Q}$, the same idea together with the notion of universal predictor can be expanded using the notion of generalized entropy associated with the loss function $\ell$, formally stated below:

 \begin{proposition}\label{thm:attainable-lower-bound}
     \textit{(Attainable lower bound in MDL)}. For MDL over a compact set of distributions $\mathcal{Q}$ with a (proper) loss function $\ell: \mathcal{Y} \times \Delta^{|\mathcal{Y}|} \rightarrow [0,M]$ that is continuous in the second argument, and a hypothesis class $\mathcal{H} = \Delta^{|\mathcal{Y}|}$, the quantity $\sup_{Q \in \mathcal{Q}}\inf_{h \in \mathcal{H}}\mathbb{E}_{Q}\left[\ell\left(\ry, h\left(\rvx\right)\right)\right]$ forms the attainable lower bound on the error of any universal predictor. 
 \end{proposition}
 \begin{proof}
     First, it is easy to check that there exists $Q^{*} \in \mathcal{Q}$ that attains the said quantity $\sup_{Q\in \mathcal{Q}}\inf_{h \in \mathcal{H}}\mathbb{E}_{Q}\left[\ell\left(\ry, h\left(\rvx\right)\right)\right]$, i.e. $\sup_{Q\in \mathcal{Q}}\inf_{h \in \mathcal{H}}\mathbb{E}_{Q}\left[\ell\left(\ry, h\left(\rvx\right)\right)\right] = \inf_{h \in \mathcal{H}}\mathbb{E}_{Q^{*}}\left[\ell\left(\ry, h\left(\rvx\right)\right)\right]$ which is the same as $H_{\ell}\left(Q^{*}\left(\ry \ \vert \ \rvx\right)\right)$---the generalized entropy of $Q^{*}$ for the loss $\ell$. The argument that this is the attainable lower bound follows from the definition of Bayes risk and the definition of universal prediction: If there exists $\bar{h} \in \mathcal{H}$ that achieves lower risk, then for $Q^{*}$: $\mathbb{E}_{Q^{*}}\left[\ell\left(\ry, \bar{h}\left(\rvx\right)\right)\right] < \inf_{h \in \mathcal{H}}\mathbb{E}_{Q^{*}}\left[\ell\left( \ry, h\left(\rvx\right)\right)\right]$. However, this is a contradiction, and one can argue that no such $\bar{h} \in \mathcal{H}$ exists. 
 \end{proof}
 The above statement is straightforward, however we are not aware of it being explicitly stated in the literature concerning MDL. We argue that the convexity of the hypothesis class is a crucial assumption here, as without the Bayesian mixture models from the minimum description length literature can achieve even lower error \citep{grunwald2007minimum}. Since the quantity  $\sup_{Q\in \mathcal{Q}}\inf_{h \in \mathcal{H}}\mathbb{E}_{Q}\left[\ell\left(\ry, h\left(\rvx\right)\right)\right]$ is the maximum generalized entropy for the loss function $\ell$ over the set $\mathcal{Q}$, the connection with the Shannon's source coding theorem is established. Furthermore, it gives way that to achieve the lower bound (or to learn optimally over $\mathcal{Q}$), learn with respect to the one distribution $Q^{*}$ that has the maximum generalized entropy for the loss function.

While the above mentioned rule gives the attainable lower bound, it is not clear if it can actually be attained. Generally, identifying the distribution with the maximum generalized entropy could be impractical. However, following \citeauthor{grunwald2004game}'s \citeyearpar{grunwald2004game} result that \textit{maximum entropy is robust Bayes}, the said lower bound can be attained by solving the $\inf_{h \in \mathcal{H}}\sup_{Q \in \mathcal{Q}}\mathbb{E}_{Q}\left[\ell\left(\ry, h\left(\rvx\right)\right)\right]$ optimization problem. Correspondingly, as done by \citet{haghtalab2022ondemand}, the MDL framework can be formulated in terms of the familiar zero-sum game between the nature and decision-maker in game-theory \citep{rockafellar1970convex}. Nature's move is to pick a distribution $Q \in \mathcal{Q}$ and the decision-maker moves by picking a hypothesis $h \in \mathcal{H}$. Nature's goal is to maximize the decision maker's risk $\mathbb{E}_{Q \in \mathcal{Q}}\left[\ell\left(\ry, h\left(\rvx\right)\right)\right]$ while the decision-maker aims to keep it as low as possible. Both players are aware of the other player's strategy but not the actual move, and hence they act conservatively. When it is nature's turn, its best move is to pick the distribution such that $Q \in \mathcal{Q}$ that maximizes the minimal risk the decision-maker can achieve across all possible hypotheses. 
% the robust Bayes act for the decision-maker $\inf_{h \in \mathcal{H}}\sup_{Q \in \mathcal{Q}}\mathbb{E}_{Q}\left[\ell\left(\ry, h\left(\rvx\right)\right)\right]$

% Furthermore, the said optimal strategy $\sup_{Q \in \mathcal{Q}}\inf_{h \in \mathcal{H}}\mathbb{E}_{Q}\left[\ell\left(\ry, h\left(\rvx\right)\right)\right]$ is not the strategy available to the decision-maker. 
In the set-up of the game, the strategy available to the decision-maker is $\inf_{h \in \mathcal{H}}\sup_{Q \in \mathcal{Q}}\mathbb{E}_{Q}\left[\ell\left(\ry, h\left(\rvx\right)\right)\right]$, and it holds that $\sup_{Q \in \mathcal{Q}}\inf_{h \in \mathcal{H}}\mathbb{E}_{Q}\left[\ell\left(\ry, h\left(\rvx\right)\right)\right] \leq \inf_{h \in \mathcal{H}}\sup_{Q \in \mathcal{Q}}\mathbb{E}_{Q}\left[\ell\left(\ry, h\left(\rvx\right)\right)\right]$. However, under the assumptions of the classical minimax theorems \citep{rockafellar1970convex, grunwald2004game}, both strategies coincide. Denoting $\mathbb{E}_{Q}\left[\ell\left(\ry, h\left(\rvx\right)\right)\right]$ as $K\left(Q, h\right)$, it is easy to see that $K\left(Q, h\right)$ is concave in $\mathcal{Q}$. Furthermore, under the standard assumptions of the loss function $\ell$ being convex in the second argument, and the continuity of the function $K$ over $\mathcal{Q} \times \mathcal{H}$, an extra assumption of $\mathcal{Q}$ and $\mathcal{H}$ being compact and convex makes the classical Von-Neumann's minimax theorem apply in the considered case. The hypothesis class $\mathcal{H}$ satisfies this assumption in the unrestricted case of modeling the full simplex $\Delta^{|\mathcal{Y}|}$. We further refer the reader to \citet{grunwald2004game} for general versions of the minimax theorems. However, without resorting to further explicate the technicalities, we assume that under an extra assumption of $\mathcal{Q}$ being compact and convex, the minimax theorem can be applied to the MDL case, and the solution $(Q^{*}, h^{*})$ has the form of a saddle-point where $Q^{*}$ attains the value $\max_{Q \in \mathcal{Q}}\inf_{h \in \mathcal{H}}K\left(Q, h\right)$ and $h^{*}$ attains the value $\min_{h \in \mathcal{H}}\sup_{Q \in \mathcal{Q}}K\left(Q, h\right)$, and saddle-point defining the best-response of each player against each other in the zero-sum game, this leads to the conclusion that the MDL problem can be optimally solved by learning with respect to the distribution $Q^{*}$ that has the maximum generalized entropy over the set $\mathcal{Q}$, and the decision-maker's strategy will attain this lower bound. We next give three standard examples of MDL framework to further elaborate on the result:

\paragraph{Distributional Robust Optimization (DRO)} The framework of DRO \citep{DRORahimian_2022} discounts the belief in the available (nominal) empirical distribution $\hat{P}_{D}=\frac{1}{N}\sum_{i=1}^{N}\delta_{\left(\vx_i, y_i\right)}$ by considering an ambiguity set of the form $\mathcal{Q} = \left\{Q \ : \ d\left(\hat{P},Q\right) \leq \epsilon\right\}$ where $d$ is some divergence, usually an $f$-divergence \citep{Rnyi1961OnMO} or Wasserstein distance \citep{Villani2009}, and the optimization problem considers the worst-case perspective with respect to $\mathcal{Q}$ as $h^{*} \in \arg\min_{h \in \mathcal{H}} \sup_{Q \in \mathcal{Q}}\mathbb{E}_{Q}\left[\ell\left(\ry, h\left(\rvx\right)\right)\right]$. The aim is to get the robust predictor by considering all the distributions that lie in the vicinity of the empirical training distribution $\hat{P}_{D}$, i.e. consider a range a distributions that might have resulted in the dataset $D$ instead of the single base measure $P$ as in the traditional setup. 

% For example, decision makers in extremely safety-critical applications desire to control the worst mistakes, and thus could employ a tail risk measure like CVaR \citep{acerbi2002coherence}. These spectral risk measures have one-to-one correspondence with the envelope representation of the form $\mathcal{Q}$---a closed and convex set of distributions \citep{Rockafellar2007CoherentAT, riskmeasuresandupperprobabilities}.  Intuitively, a decision-maker considers a range of weighing functions over the loss' distribution function as per their decision-making disposition. A risk neutral agent considers the whole distribution function uniformly, however a risk-averse agent might focus more on the extreme events.

\paragraph{Fairness / Min-Max fairness} To prevent the systemic disadvantage of the machine learning systems against some protected sub-groups \citep{buolamwini2018gender}, 
a min-max notion of fairness is proposed \citep{rothblum2021multi, pmlr-v119-martinez20a, mohri2019agnostic} to control the worst-error faced by any (sub)population. Given $\mathcal{Q}=\{Q_1, Q_2, \ldots, Q_{k}\}$ as a set of distinct populations, the goal is design a predictor $h^{*} \in \arg\min_{h \in \mathcal{H}}\max_{Q \in \mathcal{Q}}\mathbb{E}_{Q}\left[\ell\left(h_{\theta}\left(\rvx\right), \ry\right)\right]$. The rationale for min-max fairness is to directly control the worst affected population, thereby to control the maximum algorithmic disparity. Furthermore, it also guarantees that error on other sub-populations can be bounded from above.

\paragraph{Decision-Making and Risk Measures} 
A decision-maker could aim for a different decision-making perspective instead of the risk-neutral disposition of the expectations.  Formally, given $\rvz$: a random variable for some loss with some base distribution $P$ over it, the goal is to define a aggregation function $\mathbb{R}\left(\rvz\right)$ to quantify the inherent risk $\rvz$. Contrary to the standard expectations, there is an emerging trend of incorporating alternate risk measures instead of the expectation in the traditional machine learning setup \citep{pmlr-v206-mehta23b, modelbasedCVaRandbeyond}, in particular coherent risk measures \citep{coherentriskmeasures} forms an axiomatic approach to risk aggregation widely studied in finance.  It can be shown that every coherent risk measure has a unique envelope representation in the form of a closed and convex set $\mathcal{Q}$ \citep{Rockafellar2007CoherentAT}. Intuitively, given $P$, one can consider a range of density functions $\mathcal{Q}$ (or equivalently probability measures absolutely continuous   with respect to $P$), and construct $\mathbb{R}\left[\rvz\right] = \sup_{Q \in \mathcal{Q}}\mathbb{E}_{Q}\left[\rvz\right]$. $\text{CVaR}_{\alpha}$ \citep{Serraino2013, acerbi2002coherence} is a special risk measure as any other coherent risk measure can be constructed from it \citep{Kusuoka2001}, and it has gained particular attention in machine learning. An envelope (in terms of densities) for $\text{CVaR}_{\alpha}$ is given as $\mathcal{Q} = \{Q \ : \ 0 \leq Q \leq \frac{1}{1-\alpha}, \ \mathbb{E}_{P}\left[Q\right] = 1\}$. In contrast to the multiple distributions over $\mathcal{X} \times \mathcal{Y}$, $\text{CVaR}_{\alpha}$ for $\alpha \in [0,1)$, considers re-weightings upto $\frac{1}{1-\alpha}$ of the same base distribution. Intuitively, a decision-maker considers a range of weighing functions over the loss' distribution function as per their decision-making disposition. A risk neutral agent considers the whole distribution function uniformly, however a risk-averse agent might focus more on the extreme events.

\looseness=-1Thus, except for the case of min-max fairness, the stated lower bound can be attained, and those problems have a tractable optimal solution as $h^{*}\left(\vx\right) = Q^{*}\left(\ry \vert \ \rvx=\vx\right)$ where $Q^{*}$ is the one that has the maximum generalized entropy for the associated loss function $\ell$ in the set $\mathcal{Q}$. However for fairness, it is possible to consider the convex hull of the set $\mathcal{Q}$, denoted as $\text{conv}\left(\mathcal{Q}\right)$ to apply the result. However, the optimal distribution $Q^{*}$ may not be coincide with any $Q \in \mathcal{Q}$, and can lie inside the interior of the convex hull $\text{conv}\left(\mathcal{Q}\right)$. 

\fbox{\parbox{16.28cm}{\textbf{Takeaway 1:} In MDL, learning with respect to the distribution with the maximum (generalized) entropy associated with the considered loss function gives the attainable lower bound for MDL.  %Generally this bound can be attained.
}}

\section{Calibration in MDL}

As noted before, MDL framework becomes popular due to the increasing demands for robustness, fairness, and decision-making of the predictors in machine learning. However, one crucial assumption is still unaddressed. The assumption is that of the loss function $\ell$: machine learning practitioners are required to specify a loss function beforehand, and the resulting systems provide the low-risk guarantee for this specific loss function. The proposed framework of MDL falls short of translating the promise of robustness, fairness, and decision-making to arbitrary cost / utility functions. Calibration, as we argued in Section \ref{subsec:calibration-and-refinement}, is one missing piece that can help fulfill that promise. In this section, we study the notion of canonical calibration (Definition \ref{def:canonical-calibration}) in the case of MDL. In particular, we are interested in the native calibration properties as emerged out of the MDL framework, with special focus on their decision-theoretic aspects. This is generally tricky as the underlying predictor in MDL can depend on the notion of universality employed, however our analysis in the previous section (Section \ref{sec:MDL}) allow us to write the optimal predictor for MDL tractably. Hence, we take a general approach and study the calibration behavior for MDL at its attainable lower bound. We borrow heavily from Section \ref{subsec:calibration-and-refinement}, in particular the geometry of risk minimization with proper scoring losses.

Given an MDL problem for the set of distributions $\mathcal{Q}$ and a (proper) loss $\ell$, the optimal solution is given by a saddle-point $\left(Q^{*}, h^{*}\right)$ where $Q^{*}=\arg\max_{Q \in \mathcal{Q}}H_{\ell}\left(Q\left(\ry \ \vert \ \rvx\right)\right)$ and $h^{*}\left(\vx\right) = Q^{*}\left(\ry \ \vert \ \rvx=\vx\right)$, and from the saddle-point definition, it follows that $\mathbb{E}_{Q}\left[\ell\left(\ry, h^{*}\left(\rvx\right)\right)\right] \leq \mathbb{E}_{Q^{*}}\left[\ell\left(\ry, h^{*}\left(\rvx\right)\right)\right] \leq \mathbb{E}_{Q^{*}}\left[\ell\left(\ry, h\left(\rvx\right)\right)\right]$ for any $Q \in \mathcal{Q}, h \in \mathcal{H}$. For calibration for a distribution $Q \in \mathcal{Q}$, we have to analyze $Q\left(\ry \ \vert \ h^{*}\left(\rvx\right) = \nu\right)$, where $\nu \in \Delta^{|\mathcal{Y}|}$ is some prediction. It is easy to check that the canonical calibration property is satisfied for the distribution $Q^{*}$, i.e. $Q^{*}\left(\ry \ \vert \ h^{*}\left(\rvx\right) = \nu\right) = \nu$. However, our goal is to study the calibration behavior across the set of distributions $\mathcal{Q}$. We note that while choosing the Bayes optimal $Q^{*}\left(\ry \ \vert \ \rvx\right)$ certainly is not optimal for any $Q \neq Q^{*}$ and one incurs non-trivial error on $Q$ for doing so, the argument about calibration is more nuanced as calibration asks for alignment between the forecast $\nu$ and $Q\left(\ry \ \vert \ \nu\right)$, and the calibration behavior for $Q \neq Q^{*}$ does not follow from the error argument. For example, a vacuous average predictor is calibrated but is not Bayes optimal. 

% In order to give intuition for our main result in this section, we first recall the geometry implied by the proper loss minimization. As shown in Figure \ref{fig:mis-calibration-intuition}, for the forecast $\nu$, calibration error for the distribution $Q$ can intuitively be defined as the value of the tangent to the generalized entropy function at $\nu$ evaluated at $Q\left(\ry \ \vert \ \nu\right)$ minus the generalized entropy value of $Q\left(\ry \ \vert \ \nu\right)$,  and it increases as one goes further $\nu$. When the perfect calibration does not hold for some forecast $\nu$, graphically it can be deduced that it leads to the calibration-refinement trade-off for individual predictions, i.e. when the refinement of the individual prediction $\nu$ is low for $Q$ (or the discriminativeness of the prediction is high), it will be compensated by the increased calibration error for this prediction. While this is generally true for any high-confident prediction in a single-distribution case, the Bayes optimal predictor guards one against this drawback as for any $Q$, the Bayes optimal predictor is perfectly calibrated as well as has the lowest refinement possible. 
Following from the calibration-refinement decomposition in Section \ref{subsec:calibration-and-refinement} and the saddle-point property, we state two \textit{easy-to-verify} results for calibration in MDL:

\begin{proposition}\label{prop:calibration-error-bound-in-terms-of-entropy-difference}
    \textit{(Calibration error bound)}. For MDL over $\mathcal{Q}$ and the loss function $\ell$, and the optimal predictor $h^{*}:=Q^{*}\left(\ry \ \vert \ \rvx\right)$ with the maximum generalized entropy, the calibration error for any distribution $Q \in \mathcal{Q}$ is bounded as below:
    \begin{align*}
    \begin{split}
    \mathbb{E}_{Q}\left[d_{\ell}\left(Q\left(\ry \ \vert \ h^{*}\left(\rvx\right)\right), \ h^{*}\left(\rvx\right)\right)\right]  \leq \mathbb{E}_{Q^{*}}\left[H_{\ell}\left(Q^{*}\left(\ry \ \vert \ \rvx\right)\right)\right] - \mathbb{E}_{Q}\left[H_{\ell}\left(Q\left(\ry \ \vert \ \rvx\right)\right)\right].
    \end{split}
    \end{align*}
    Furthermore, barring any distributional assumptions between $Q$ and $Q^{*}$, the predictor $h^{*}\left(\vx\right)=Q^{*}\left(\ry \ \vert  \ \rvx = \vx\right)$ cannot be perfectly canonically calibrated for $Q$. 
\end{proposition}

\begin{proof}
    From the calibration-refinement decomposition and the property of the saddle-point, it follows that 
 \begin{align*}
    \begin{split}
    \mathbb{E}_{Q}\left[d_{\ell}\left(Q\left(\ry \ \vert \ h^{*}\left(\rvx\right)\right), \ h^{*}\left(\rvx\right)\right)\right]  \leq \mathbb{E}_{Q^{*}}\left[H_{\ell}\left(Q^{*}\left(\ry \ \vert \ \rvx\right)\right)\right] - \mathbb{E}_{Q}\left[H_{\ell}\left(Q\left(\ry \ \vert \ h^{*}\left(\rvx\right)\right)\right)\right],
    \end{split}
    \end{align*}
    and since $\mathbb{E}_{Q}\left[H_{\ell}\left(Q\left(\ry \ \vert \ h^{*}\left(\rvx\right)\right)\right)\right] \geq \mathbb{E}_{Q}\left[H_{\ell}\left(Q\left(\ry \ \vert \ \rvx \right)\right)\right]$, the upper bound on the calibration error for $Q$ follows. Furthermore, assume that the calibration error in the above expression is zero. Then, it follows that $\mathbb{E}_{Q}\left[H_{\ell}\left(Q\left(\ry \ \vert \ h^{*}\left(\rvx\right)\right)\right)\right] \leq \mathbb{E}_{Q^{*}}\left[H_{\ell}\left(Q^{*}\left(\ry \ \vert \ \rvx\right)\right)\right]$, or $\mathbb{E}_{Q}\left[H_{\ell}\left(Q^{*}\left(\ry \ \vert \ \rvx\right)\right)\right] \leq \mathbb{E}_{Q^{*}}\left[H_{\ell}\left(Q^{*}\left(\ry \ \vert \ \rvx\right)\right)\right]$, which gives the relation over the distributions $Q$ and $Q^{*}$. Furthermore, in the case when this holds with equality, the predictor is maximally uninformative for the distribution $Q$ as $\mathbb{E}_{Q^{*}}\left[H_{\ell}\left(Q^{*}\left(\ry \ \vert \ \rvx\right)\right)\right]$ has the maximum generalized entropy for the loss function $\ell$ in the envelope $\mathcal{Q}$. 
    % If these two expressions are equal, then $Q=Q^{*}$. Since $Q \neq Q^{*}$, the only possibility is $\mathbb{E}_{Q}\left[H_{\ell}\left(Q^{*}\left(\ry \ \vert \ \rvx\right)\right)\right] < \mathbb{E}_{Q^{*}}\left[H_{\ell}\left(Q^{*}\left(\ry \ \vert \ \rvx\right)\right)\right]$. However, this results in $Q\left(\rvx\right) < Q^{*}\left(\rvx\right)$ for significant $\rvx_{s}$. If this assumption does not hold true, then perfect canonical calibration cannot hold true for any $Q \neq Q^{*}$. 
\end{proof}
To further elaborate, consider the risk prediction scenario over the hypothetical population where every member has a $50-50$ chance of having certain risk. Then the optimal predictor in the one which outputs $0.5$ confidence for every member. For this predictor to be perfectly calibrated on some other population, the marginal label distribution for this population also needs to be $50-50$.  

\begin{corollary}\label{coro:non-uniform-disparity}
    \textit{(Non-uniform disparity)}. The calibration errors over the set of distributions $\mathcal{Q}$ form a non-increasing sequence converging to $0$, i.e. 
    $$\lim_{Q \to Q^{*}}\mathbb{E}\left[d_{\ell}\left(Q\left(\ry \ | \ h^{*}\left(\rvx\right)\right), h^{*}\left(\rvx\right)\right)\right] = 0.$$
\end{corollary}

\begin{wrapfigure}{r}{0.5\linewidth}
    \centering
    \includegraphics[width=0.82\linewidth, height=6.5cm]{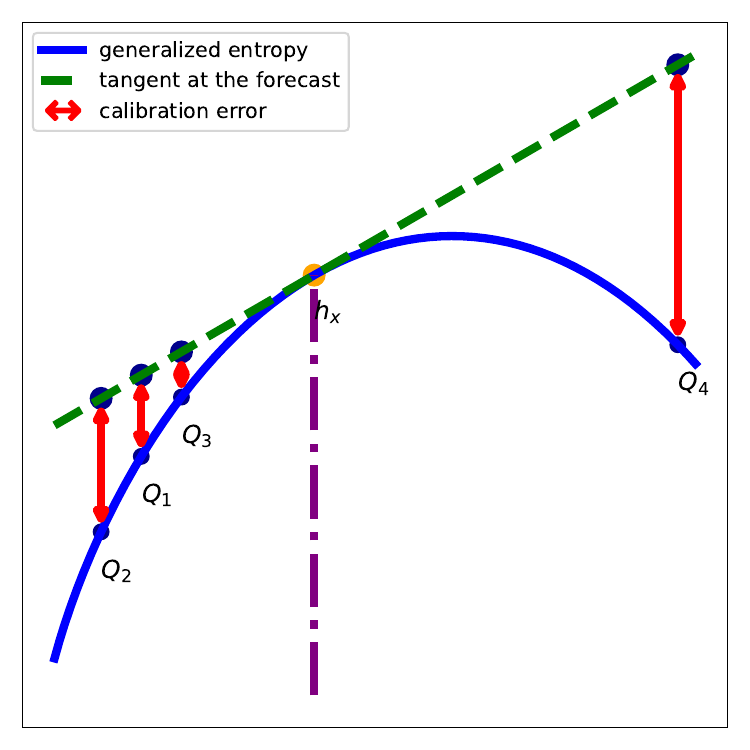}
    \caption{\textit{Calibration disparity intuition in MDL:} For a forecast $h_{\vx}$, the calibration error is defined by the (generalized) entropy function $H_{\ell}$ as a hyperplane at $h_{\vx}$ evaluated at $Q\left(\ry \ \vert \ h_{\vx}\right)$ minus $H_{\ell}\left(Q\left(\ry \ \vert \ h_{\vx}\right)\right)$.}
    \label{fig:mis-calibration-intuition}
\end{wrapfigure}

To summarize, by the definition of $Q^{*}$, the upper-bound on the calibration error is positive, and under only the distributional assumptions the calibration error can be perfectly assumed to be zero. Thus, the generalized entropy difference in $\mathcal{Q}$ controls the calibration error, and the calibration errors are non-uniformly distributed for different distributions in $\mathcal{Q}$. Given that calibration error for any $Q \neq Q^{*}$ can only be assumed to be zero under certain assumptions,
% condition of $Q^{*}\left(\rvx\right)$ strongly dominating $Q\left(\rvx\right)$\footnote{To reduce technicalities, we informally say a distribution $Q^{*}$ strongly dominates $Q$ when $Q^{*}$ puts higher mass than $Q$ for a major portion of the support of these two distributions. This definition can be made formal by quantifying the subset of the support where \( Q^{*}(A) > Q(A) \) and requiring that the total mass assigned by \( Q^{*} \) exceeds that of \( Q \) over this subset.}, 
the geometry of proper loss minimization gives additional perspective on the calibration error.  As shown in Figure \ref{fig:mis-calibration-intuition}, for the forecast $\nu$, calibration error for the distribution $Q$ can intuitively be defined as the value of the tangent to the generalized entropy function at $\nu$ evaluated at $Q\left(\ry \ \vert \ \nu\right)$ minus the generalized entropy value of $Q\left(\ry \ \vert \ \nu\right)$,  and it increases as one goes further $\nu$. When the perfect calibration does not hold for some forecast $\nu$, graphically it can be deduced that it leads to the calibration-refinement trade-off for individual predictions, i.e. when the refinement of the individual prediction $\nu$ is low for $Q$ (or the discriminativeness of the prediction is high), it will be compensated by the increased calibration error for this prediction. While this is generally true for any high-confident prediction in a single-distribution case, the Bayes optimal predictor guards one against this drawback as for any $Q$, the Bayes optimal predictor is perfectly calibrated as well as has the lowest refinement possible. However, we assert this is a fundamental artifact of MDL. We state it (in)formally below:

\begin{corollary}
    There is a fundamental calibration-refinement trade-off in MDL, even at optimality.  Furthermore, a prediction has different meaning for different distributions. 
\end{corollary}
\begin{proof}
    For a distribution $Q$, the minimal error is $a:=\mathbb{E}_{Q}\left[H_{\ell}\left(Q\left(\ry \ \vert \ \rvx\right)\right)\right]$. And the error of the predictor in MDL is $b:=\mathbb{E}_{Q}\left[\ell\left(\ry, h^{*}\left(\rvx\right)\right)\right]$ which decomposes as $\mathbb{E}_{Q}\left[d_{\ell}\left(Q\left(\ry \ \vert \ h^{*}\left(\rvx\right), h^{*}\left(\rvx\right)\right)\right)\right] + \mathbb{E}_{Q}\left[H_{\ell}\left(Q\left(\ry\ \vert \ h^{*}\left(\rvx\right)\right)\right)\right]$. Since $a < b$. If the predictor becomes more informative than the Bayes predictor, it will be compensated by increase in the calibration error to maintain the inequality. The last statement in the above result follows trivially from the calibration-disparity result.
\end{proof}
\looseness=-1 One might wonder what the significance of this calibration-refinement trade-off is, especially since MDL, by design, controls the worst-case scenario. However, an interesting question is whether the worst-case guarantee for the loss function $\ell$ translates to arbitrary cost functions of an arbitrary decision-maker? The next statement says no:  

\begin{proposition}\label{prop:multi-distribution-learning-calibration-and-decision-theoretic-property}
    \textit{(Calibration and decision-making)}. Given $\mathcal{Q}$ and a predictor $h$ calibrated with respect to $Q^{*} \in \mathcal{Q}$ with the maximum generalized entropy for a loss function $\ell$, a decision rule $\delta : h\left(\rvx\right) \mapsto \arg\min_{a \in \mathcal{A}}\mathbb{E}_{\ry \sim h\left(\rvx\right)}\left[c\left(a, \ry\right)\right]$ with the action space $\mathcal{A}$ and a cost function $c: \mathcal{A} \times \mathcal{Y} \rightarrow \mathbb{R}_{+}$ is optimal in the worst case sense over $\mathcal{Q}$ as long as the cost function $c$ is consistent with the loss function $\ell$.
\end{proposition}
% Proof \ref{sec:proof-MDL-calibration-and-decision-theoretic-property} 
The proof is in Section \ref{sec:decision-theoretic-calibration-proofs-section} together with the proof of Proposition \ref{prop:calibration-and-decision-making-regular-calibration}. In contrast to the regular calibration (Proposition \ref{prop:calibration-and-decision-making-regular-calibration}), a decision-maker in MDL is limited by the kind of cost functions they can consider to successfully exploit the worst-case guarantee provided by the framework. Although these are simple results, we argue that they might have profound implications, in particular for decision-making. In fairness, for example, the lack of uniform calibration in $\mathcal{Q}$ necessitates further post-processing of the prediction for different populations depending on the miscalibration. This might require identifiability of the distribution, auditing for calibration errors, and correcting it.

\paragraph{Guide to Practitioners} We next discuss the relevance of our results for practice. We agree that calibration might not be required in every applications. For example, in a fault tolerance system, if the institution only cares about a certain fault cost function, then the MDL framework would guarantee the institution against the worst-case scenario. However, in the case of general application cases like healthcare where a medical professional has to reason about arbitrary cost / utility functions, calibration of the predictor becomes an underlying requirement to bridge the institutional separation between the training time loss function and the decision-time cost / utility function. In particular, consider the motivated use-case of MDL where several different healthcare facilities jointly learn a single predictor using the MDL framework to individually allocate the decisions. In this scenario, our results indicate a critical limitation. For one, calibration for each healthcare facility is not guaranteed, and secondly, the miscalibration errors can be non-uniform leading to a different interpretation of the same prediction for each facility. Albeit our results do not give an informative lower bound for the calibration error for each distribution, they still inform that care must be taken to equitably use the predictor for arbitrary decisions. Besides post-processing on the decision-makers side, designers can also opt for directly minimizing the upper-bound for the calibration error in Proposition \ref{prop:calibration-error-bound-in-terms-of-entropy-difference} (see section \ref{sec:consequences}: DRO), or certify that the overall error $\mathbb{E}_{Q}\left[\ell\left(\ry, h^{*}\left(\rvx\right)\right)\right]$ is significantly less for each $Q \neq Q^{*}$.

\fbox{\parbox{16.28cm}{\textbf{Takeaway 2:} We have discovered a fundamental calibration-refinement trade-off in the MDL framework. This trade-off is determined by the heterogeneity in terms of the (generalized) entropy in the considered envelope of distributions. Furthermore, learning (optimally) over multiple distributions also does not truly bridge the institutional separation between model designers and decision-makers. A decision-maker is always constrained by the types of cost functions they can consider when exploiting MDL.}}

\section{Related Work}
\paragraph{MDL} \looseness=-1Learning under multiple distributions in machine learning has long been studied independently across different research themes and focus, with representative examples being DRO \citep{DRORahimian_2022} with applications in machine learning for arbitrary distribution shifts \citep{ben2013robust, duchi2021statistics}, domain adaptation \citep{ben2003exploiting, mansour2008domain} collaborative learning to design a common predictive model to achieve lower error for different stakeholders \citep{blum2017collaborative, blum2021communication, nguyen2018improved}, and fairness \citep{rothblum2021multi, pmlr-v119-martinez20a}. Broadly, MDL aims to combat the disparate impact caused by the conventional average guarantees of single distribution learning setup, also referred to as \textit{representation disparity} \citep{hashimoto2018fairness}. A recent like of work \citep{frohlich2024data, frohlich2024risk} argues for MDL as a foundational approach to practical machine learning primarily motivated by the statistical stability problem \citep{Gorban2017}: informally, it is ``naive'' to argue that the statistical regularities in the dataset can be explained by one base distribution.  Hence, the \textit{representation disparity} is a consequence of mis-specified data-driven inference setup.  Independent of different motivations, a common unification under MDL is studied by \citet{haghtalab2022ondemand, zhang2024optimal} with the focus on sample complexity and designing optimal algorithms to learn efficiently. These works extend the agnostic learning setup \citep{valiant1984theory, kearns1992toward} for learning across multiple distributions in the worst-case. Our work can also be seen as a unified study of calibration in MDL frameworks extending the celebrated calibration-refinement results for proper-scoring losses in the single distribution learning setup \citep{Degroot1983TheCA, brocker2009reliability, 10.1007/978-3-319-23528-8_5}. 

In concurrent work, \citet{frohlich2024scoring} study the calibration of imprecise forecasts, and their results bear a strong resemblance with our work. In particular, \citet{frohlich2024scoring} also use the generalized entropy arguments to derive the optimal forecast in the imprecise case as the one with the maximum expected generalized entropy. For DRO, their theoretical results further assert the importance of the training time loss function, and demonstrates how the institutional separation is not truly bridged for arbitrary cost functions at test time. They also formalize the notion of calibration of the imprecise forecast by drawing connections to actuarial fairness, i.e. a forecast is deemed to be calibrated if it enables accurate loss estimation for the decision-maker. For the imprecise forecast, they extend it to accurate estimation of the worst-case loss. In contrast, the setting of MDL we consider involves learning over a set of distributions, with the resulting predictor always considered as precise. Moreover, we concern ourselves with calibration behavior across individual distributions in MDL. 

\paragraph{Relation to Multi-Calibration and Omnipredictors}  Our findings discover a calibration disparity across distributions in the MDL framework. In the single distribution setup, calibration disparity have been widely documented \citep{obermeyer19}, and the literature in multi-calibration \citep{pmlr-v80-hebert-johnson18a} aims to overcome those by post-processing predictors to achieve calibration over (possibly infinite) intersecting subgroups, and the related notion of omniprediction \citep{gopalan2021omnipredictors} provides loss-oblivious learning guarantees that build on top of multi-calibration. However, we assert that the notion of multi-calibration assumes a fixed base distribution $P$ over $\mathcal{X} \times \mathcal{Y}$, and algorithmically multi-calibration is achieved by post-processing the predictor to get it close to the Bayes predictor $P\left(\ry \ \vert \ \rvx\right)$ with respect to $P$. However, the calibration disparity in multi-distribution is more nuanced as the Bayes optimal predictor, itself, is not guaranteed to be perfectly calibrated across all the distributions.

\section{Consequences}\label{sec:consequences}
In this section, we pay special attention to two critical applications where MDL frameworks are prominent, and discuss the consequences of our findings.

\paragraph{Distributional Robust Optimization}
As stated before, the framework of DRO \citep{DRORahimian_2022} discounts one's beliefs in the available base (empirical) distribution $\hat{P}$, and guard against possible mis-specifications, by considering a range of possible scenarios in the form of an ambiguity set. The goal is to achieve robust and generalizable decision-making under uncertainty. \textit{A crucial question then is what kind of ambiguity sets one should consider, and how do these sets influence generalization in decision-making?} For example, a machine learning designer aiming to build a risk-prediction system for some healthcare facility with some available dataset $\hat{P}$ may consider a range of settings $\mathcal{Q}$ to account for possible data imputations, for instance. A medical facility will then use the same system to allocate arbitrary decisions: whether to administer or deny treatment, or to administer treatment $A$ or $B$, e.g. by thresholding the estimated risk. The presence of calibration disparity across $\mathcal{Q}$, however, then puts the medical facility at the risk of \textit{inequitable outcomes} or misinformed decisions \textit{systemically} for different scenarios. How can the medical facility then guarantee robust as well as equitable decisions? 

\begin{corollary}\label{corr:DRO-divergence-corollary}
    For the empirical base distribution $\hat{P}$, an ambiguity set $\mathcal{Q} = \left\{Q \ : \ d\left(\hat{P}, Q\right) \leq \epsilon\right\}$ where the divergence $d$ is $k$-Lipschitz continuous with respect to the generalized entropy function of the considered loss function guarantees robust and equitable decisions.
\end{corollary}

The argument follows from the Lipschitz continuity of the divergence with respect to the generalized entropy function. In particular, $|H_{\ell}\left(P\right) - H_{\ell}\left(Q\right)| \leq 
\kappa \cdot d\left(P,Q\right)$. Then if $d\left(P,Q\right) \leq \epsilon$, one can bound the entropy difference in the envelope as $|H_{\ell}\left(P\right) - H_{\ell}\left(Q\right)| \leq \kappa \cdot \epsilon$. Intuitively, a machine learning designer can adopt a suitable divergence metric to control the (generalized) entropy difference within the envelope, thereby controlling the calibration disparity. Choosing a divergence metric is a design-choice, and commonly an $f$-divergence \citep{Rnyi1961OnMO} or Wasserstein distance measures \citep{Villani2009} are employed. While the Wasserstein distance measure is popular due to its tractable dual representation, $f$-divergences provide a flexible alternative, as their sensitivity can be tailored through an appropriately designed convex function to achieve the desired continuity property with respect to the (generalized) entropy function. Moreover, every proper scoring loss function is be associated with a Bregman divergence that satisfies this continuity property. For example, in case of the given risk-prediction task, a machine learning designer can consider a KL divergence, as it can be shown to be Lipschitz continuous with respect to the Shannon entropy function---a (generalized) entropy function associated with the negative log loss.

\paragraph{Fairness / Min-Max Fairness} As stated before, for min-max notion for fairness when there are $k$ distinct distributions, the optimal predictor may lie inside the convex hull of the set of these distributions. This implies that the resulting predictor is not calibrated with respect to any of the $k$ distributions, necessitating corrective measures. Assuming sample access to the distributions, the said predictor can then be corrected by simple post-processing measures like temperature-scaling \citep{guo2017calibration} can be used to fix the calibration errors. However, this introduces extra optimization and sample access complexity. 

% \begin{corollary}
%     \textit{(Calibration enables optimal loss estimation)}. 
% \end{corollary}

\section{Discussion}
\paragraph{Conclusions}We have presented a general theoretical analysis of calibration disparity in the context of multi-distribution learning. By focusing on the relationship between generalized entropy and calibration error bounds, we provide insights into how calibration disparity arises in models optimized over multiple distributions. We highlight the consequences of our findings to decision-making frameworks, such as DRO and fairness, demonstrating the implications for robustness and fairness in real-world applications. Our results provide a foundation for future research on designing more robust and well-calibrated models under distributional uncertainty. 

\paragraph{Limitations and Future Work} Our focus is a theoretical characterization, which sidesteps specific algorithms for MDL and their effects on calibration. Furthermore, we study calibration in the batch setting, and specifically distributional calibration---the most strongest form of calibration \citep{pmlr-v89-vaicenavicius19a}. In the future, our goal is to study the implications of our findings empirically in practical settings to fully characterize the findings. While we highlight the calibration-refinement generally, in DRO a specific choice of divergence metric can reduce that effect. Our work aims to give a constructive guideline on the choice of the divergence metric, however, that choice is also limited by the ease of optimization. Hence, it is interesting to explore alternative notion of divergences that reduce the calibration disparity, while also leading to computationally tractable practical algorithms. Our work highlights potential issues for calibration to realize the full potential of MDL and informs practical design choices to reduce those issues. 

\section{Acknowledgments}
We thank Alexander Timans, Daan Roos, and Patrick Forré for helpful discussions. UvA-Bosch Delta Lab at the University of Amsterdam is funded by the Bosch Center for Artificial Intelligence.

\bibliography{main}

\begin{thebibliography}{61}
\providecommand{\natexlab}[1]{#1}
\providecommand{\url}[1]{\texttt{#1}}
\expandafter\ifx\csname urlstyle\endcsname\relax
  \providecommand{\doi}[1]{doi: #1}\else
  \providecommand{\doi}{doi: \begingroup \urlstyle{rm}\Url}\fi

\bibitem[Acerbi \& Tasche(2002)Acerbi and Tasche]{acerbi2002coherence}
Carlo Acerbi and Dirk Tasche.
\newblock On the coherence of expected shortfall, 2002.

\bibitem[Arjovsky(2020)]{arjovsky2020out}
Martin Arjovsky.
\newblock \emph{Out of distribution generalization in machine learning}.
\newblock PhD thesis, New York University, 2020.

\bibitem[Artzner et~al.(1999)Artzner, Delbaen, Eber, and Heath]{coherentriskmeasures}
Philippe Artzner, Freddy Delbaen, Jean-Marc Eber, and David Heath.
\newblock Coherent measures of risk.
\newblock \emph{Mathematical Finance}, 1999.

\bibitem[Ben-David \& Schuller(2003)Ben-David and Schuller]{ben2003exploiting}
Shai Ben-David and Reba Schuller.
\newblock Exploiting task relatedness for multiple task learning.
\newblock In \emph{Learning Theory and Kernel Machines: Conference on Learning Theory and Kernel Workshop, COLT/Kernel}, 2003.

\bibitem[Ben-Tal et~al.(2009)Ben-Tal, Ghaoui, and Nemirovski]{Ben-Tal-RobustOptimization}
A.~Ben-Tal, L.~El Ghaoui, and A.~Nemirovski.
\newblock \emph{Robust Optimization}.
\newblock Princeton University Press, 2009.

\bibitem[Ben-Tal et~al.(2013)Ben-Tal, Den~Hertog, De~Waegenaere, Melenberg, and Rennen]{ben2013robust}
Aharon Ben-Tal, Dick Den~Hertog, Anja De~Waegenaere, Bertrand Melenberg, and Gijs Rennen.
\newblock Robust solutions of optimization problems affected by uncertain probabilities.
\newblock \emph{Management Science}, 2013.

\bibitem[Blanchard et~al.(2011)Blanchard, Lee, and Scott]{blanchard2011generalizing}
Gilles Blanchard, Gyemin Lee, and Clayton Scott.
\newblock Generalizing from several related classification tasks to a new unlabeled sample.
\newblock \emph{Advances in Neural Information Processing Systems}, 2011.

\bibitem[Blanchard et~al.(2021)Blanchard, Deshmukh, Dogan, Lee, and Scott]{blanchard2021domain}
Gilles Blanchard, Aniket~Anand Deshmukh, Urun Dogan, Gyemin Lee, and Clayton Scott.
\newblock Domain generalization by marginal transfer learning.
\newblock \emph{Journal of Machine Learning Research}, 2021.

\bibitem[Blum et~al.(2017)Blum, Haghtalab, Procaccia, and Qiao]{blum2017collaborative}
Avrim Blum, Nika Haghtalab, Ariel~D Procaccia, and Mingda Qiao.
\newblock Collaborative pac learning.
\newblock \emph{Advances in Neural Information Processing Systems}, 2017.

\bibitem[Blum et~al.(2021)Blum, Heinecke, and Reyzin]{blum2021communication}
Avrim Blum, Shelby Heinecke, and Lev Reyzin.
\newblock Communication-aware collaborative learning.
\newblock In \emph{AAAI Conference on Artificial Intelligence}, 2021.

\bibitem[Bregman(1967)]{bregman1967relaxation}
Lev~M Bregman.
\newblock The relaxation method of finding the common point of convex sets and its application to the solution of problems in convex programming.
\newblock \emph{USSR computational mathematics and mathematical physics}, 1967.

\bibitem[Br{\"o}cker(2009)]{brocker2009reliability}
Jochen Br{\"o}cker.
\newblock Reliability, sufficiency, and the decomposition of proper scores.
\newblock \emph{Quarterly Journal of the Royal Meteorological Society: A journal of the atmospheric sciences, applied meteorology and physical oceanography}, 2009.

\bibitem[Buolamwini \& Gebru(2018)Buolamwini and Gebru]{buolamwini2018gender}
Joy Buolamwini and Timnit Gebru.
\newblock Gender shades: Intersectional accuracy disparities in commercial gender classification.
\newblock In \emph{Conference on Fairness, Accountability and Transparency}, 2018.

\bibitem[Dawid(1982)]{Dawid1982TheWB}
A.~Philip Dawid.
\newblock The well-calibrated bayesian.
\newblock \emph{Journal of the American Statistical Association}, 1982.

\bibitem[Degroot \& Fienberg(1983)Degroot and Fienberg]{Degroot1983TheCA}
Morris~H. Degroot and Stephen~E. Fienberg.
\newblock The comparison and evaluation of forecasters.
\newblock \emph{The Statistician}, 1983.

\bibitem[Derr \& Williamson(2024)Derr and Williamson]{derr2024four}
Rabanus Derr and Robert~C Williamson.
\newblock Four facets of forecast felicity: Calibration, predictiveness, randomness and regret.
\newblock \emph{arXiv preprint arXiv:2401.14483}, 2024.

\bibitem[Duchi et~al.(2021)Duchi, Glynn, and Namkoong]{duchi2021statistics}
John~C Duchi, Peter~W Glynn, and Hongseok Namkoong.
\newblock Statistics of robust optimization: A generalized empirical likelihood approach.
\newblock \emph{Mathematics of Operations Research}, 2021.

\bibitem[Dwork et~al.(2021)Dwork, Kim, Reingold, Rothblum, and Yona]{dwork2021outcome}
Cynthia Dwork, Michael~P Kim, Omer Reingold, Guy~N Rothblum, and Gal Yona.
\newblock Outcome indistinguishability.
\newblock In \emph{ACM SIGACT Symposium on Theory of Computing}, 2021.

\bibitem[Eastwood et~al.(2022)Eastwood, Robey, Singh, Von~K{\"u}gelgen, Hassani, Pappas, and Sch{\"o}lkopf]{eastwood2022probable}
Cian Eastwood, Alexander Robey, Shashank Singh, Julius Von~K{\"u}gelgen, Hamed Hassani, George~J Pappas, and Bernhard Sch{\"o}lkopf.
\newblock Probable domain generalization via quantile risk minimization.
\newblock \emph{Advances in Neural Information Processing Systems}, 2022.

\bibitem[Elkan(2001)]{10.5555/1642194.1642224}
Charles Elkan.
\newblock The foundations of cost-sensitive learning.
\newblock In \emph{International Joint Conference on Artificial Intelligence}, 2001.

\bibitem[Fr{\"o}hlich \& Williamson(2024{\natexlab{a}})Fr{\"o}hlich and Williamson]{frohlich2024data}
Christian Fr{\"o}hlich and Robert~C Williamson.
\newblock Data models with two manifestations of imprecision.
\newblock \emph{arXiv preprint arXiv:2404.09741}, 2024{\natexlab{a}}.

\bibitem[Fr{\"o}hlich \& Williamson(2024{\natexlab{b}})Fr{\"o}hlich and Williamson]{frohlich2024risk}
Christian Fr{\"o}hlich and Robert~C Williamson.
\newblock Risk measures and upper probabilities: Coherence and stratification.
\newblock \emph{Journal of Machine Learning Research}, 2024{\natexlab{b}}.

\bibitem[Fr{\"o}hlich \& Williamson(2024{\natexlab{c}})Fr{\"o}hlich and Williamson]{frohlich2024scoring}
Christian Fr{\"o}hlich and Robert~C Williamson.
\newblock Scoring rules and calibration for imprecise probabilities.
\newblock \emph{arXiv preprint arXiv:2410.23001}, 2024{\natexlab{c}}.

\bibitem[Fr{{\"o}}hlich \& Williamson(2024)Fr{{\"o}}hlich and Williamson]{riskmeasuresandupperprobabilities}
Christian Fr{{\"o}}hlich and Robert~C. Williamson.
\newblock Risk measures and upper probabilities: Coherence and stratification.
\newblock \emph{Journal of Machine Learning Research}, 2024.

\bibitem[Gneiting \& Raftery(2007)Gneiting and Raftery]{properscoringrulesGneiting}
Tilmann Gneiting and Adrian~E Raftery.
\newblock Strictly proper scoring rules, prediction, and estimation.
\newblock \emph{Journal of the American Statistical Association}, 2007.

\bibitem[Gopalan et~al.(2021)Gopalan, Kalai, Reingold, Sharan, and Wieder]{gopalan2021omnipredictors}
Parikshit Gopalan, Adam~Tauman Kalai, Omer Reingold, Vatsal Sharan, and Udi Wieder.
\newblock Omnipredictors.
\newblock \emph{arXiv preprint arXiv:2109.05389}, 2021.

\bibitem[Gorban(2017)]{Gorban2017}
Igor~I. Gorban.
\newblock \emph{The Statistical Stability Phenomenon}.
\newblock Springer Cham, 2017.

\bibitem[Gr{\"u}nwald(2018)]{grunwald2018safe}
Peter Gr{\"u}nwald.
\newblock Safe probability.
\newblock \emph{Journal of Statistical Planning and Inference}, 2018.

\bibitem[Gr{\"u}nwald(2007)]{grunwald2007minimum}
Peter~D Gr{\"u}nwald.
\newblock \emph{The minimum description length principle}.
\newblock 2007.

\bibitem[Gr{\"u}nwald \& Dawid(2004)Gr{\"u}nwald and Dawid]{grunwald2004game}
Peter~D Gr{\"u}nwald and A~Philip Dawid.
\newblock Game theory, maximum entropy, minimum discrepancy and robust bayesian decision theory.
\newblock 2004.

\bibitem[Guo et~al.(2017)Guo, Pleiss, Sun, and Weinberger]{guo2017calibration}
Chuan Guo, Geoff Pleiss, Yu~Sun, and Kilian~Q Weinberger.
\newblock On calibration of modern neural networks.
\newblock In \emph{International Conference on Machine Learning}, 2017.

\bibitem[Haghtalab et~al.(2022)Haghtalab, Jordan, and Zhao]{haghtalab2022ondemand}
Nika Haghtalab, Michael Jordan, and Eric Zhao.
\newblock On-demand sampling: Learning optimally from multiple distributions.
\newblock \emph{Advances in Neural Information Processing Systems}, 2022.

\bibitem[H\'{a}jek(2007)]{Hajek2007-HJETRC}
Alan H\'{a}jek.
\newblock The reference class problem is your problem too.
\newblock \emph{Synthese}, 2007.

\bibitem[Hashimoto et~al.(2018)Hashimoto, Srivastava, Namkoong, and Liang]{hashimoto2018fairness}
Tatsunori Hashimoto, Megha Srivastava, Hongseok Namkoong, and Percy Liang.
\newblock Fairness without demographics in repeated loss minimization.
\newblock In \emph{International Conference on Machine Learning}, 2018.

\bibitem[Hebert-Johnson et~al.(2018)Hebert-Johnson, Kim, Reingold, and Rothblum]{pmlr-v80-hebert-johnson18a}
Ursula Hebert-Johnson, Michael Kim, Omer Reingold, and Guy Rothblum.
\newblock Multicalibration: Calibration for the ({C}omputationally-identifiable) masses.
\newblock In \emph{International Conference on Machine Learning}, 2018.

\bibitem[H{\"o}ltgen \& Williamson(2023)H{\"o}ltgen and Williamson]{holtgen2023richness}
Benedikt H{\"o}ltgen and Robert~C Williamson.
\newblock On the richness of calibration.
\newblock In \emph{ACM Conference on Fairness, Accountability, and Transparency}, 2023.

\bibitem[Kearns et~al.(1992)Kearns, Schapire, and Sellie]{kearns1992toward}
Michael~J Kearns, Robert~E Schapire, and Linda~M Sellie.
\newblock Toward efficient agnostic learning.
\newblock In \emph{Workshop on Computational Learning Theory}, 1992.

\bibitem[Kleinberg et~al.(2023)Kleinberg, Leme, Schneider, and Teng]{kleinberg2023u}
Bobby Kleinberg, Renato~Paes Leme, Jon Schneider, and Yifeng Teng.
\newblock U-calibration: Forecasting for an unknown agent.
\newblock In \emph{Conference on Learning Theory}, 2023.

\bibitem[Kull \& Flach(2015)Kull and Flach]{10.1007/978-3-319-23528-8_5}
Meelis Kull and Peter Flach.
\newblock Novel decompositions of proper scoring rules for classification: score adjustment as precursor to calibration.
\newblock 2015.

\bibitem[Kusuoka(2001)]{Kusuoka2001}
Shigeo Kusuoka.
\newblock \emph{On law invariant coherent risk measures}.
\newblock Springer Japan, 2001.

\bibitem[Mansour et~al.(2008)Mansour, Mohri, and Rostamizadeh]{mansour2008domain}
Yishay Mansour, Mehryar Mohri, and Afshin Rostamizadeh.
\newblock Domain adaptation with multiple sources.
\newblock In \emph{Advances in Neural Information Processing Systems}, 2008.

\bibitem[Martinez et~al.(2020)Martinez, Bertran, and Sapiro]{pmlr-v119-martinez20a}
Natalia Martinez, Martin Bertran, and Guillermo Sapiro.
\newblock Minimax pareto fairness: A multi objective perspective.
\newblock In \emph{International Conference on Machine Learning}, 2020.

\bibitem[Mehta et~al.(2023)Mehta, Roulet, Pillutla, Liu, and Harchaoui]{pmlr-v206-mehta23b}
Ronak Mehta, Vincent Roulet, Krishna Pillutla, Lang Liu, and Zaid Harchaoui.
\newblock Stochastic optimization for spectral risk measures.
\newblock In \emph{International Conference on Artificial Intelligence and Statistics}, 2023.

\bibitem[Meng \& Gower(2023)Meng and Gower]{modelbasedCVaRandbeyond}
Si~Yi Meng and Robert~M. Gower.
\newblock A model-based method for minimizing cvar and beyond.
\newblock In \emph{International Conference on Machine Learning}, 2023.

\bibitem[Mohri et~al.(2019)Mohri, Sivek, and Suresh]{mohri2019agnostic}
Mehryar Mohri, Gary Sivek, and Ananda~Theertha Suresh.
\newblock Agnostic federated learning.
\newblock In \emph{International Conference on Machine Learning}, 2019.

\bibitem[Nguyen \& Zakynthinou(2018)Nguyen and Zakynthinou]{nguyen2018improved}
Huy Nguyen and Lydia Zakynthinou.
\newblock Improved algorithms for collaborative pac learning.
\newblock In \emph{Advances in Neural Information Processing Systems}, 2018.

\bibitem[Noarov \& Roth(2024)Noarov and Roth]{let4allcalibrationblog}
Georgy Noarov and Aaron Roth.
\newblock Calibration for decision making: A principled approach to trustworthy ml, 2024.
\newblock Blog post.

\bibitem[Obermeyer et~al.(2019)Obermeyer, Powers, Vogeli, and Mullainathan]{obermeyer19}
Ziad Obermeyer, Brian Powers, Christine Vogeli, and Sendhil Mullainathan.
\newblock Dissecting racial bias in an algorithm used to manage the health of populations.
\newblock \emph{Science}, 2019.

\bibitem[Ovcharov(2018)]{ovcharov2018proper}
Evgeni~Y Ovcharov.
\newblock Proper scoring rules and bregman divergence.
\newblock 2018.

\bibitem[Rahimian \& Mehrotra(2022)Rahimian and Mehrotra]{DRORahimian_2022}
Hamed Rahimian and Sanjay Mehrotra.
\newblock Frameworks and results in distributionally robust optimization.
\newblock \emph{Open Journal of Mathematical Optimization}, 2022.

\bibitem[R{\'e}nyi(1961)]{Rnyi1961OnMO}
Alfr{\'e}d R{\'e}nyi.
\newblock On measures of entropy and information.
\newblock 1961.

\bibitem[Rockafellar(1970)]{rockafellar1970convex}
RT~Rockafellar.
\newblock Convex analysis.
\newblock \emph{Princeton Mathematical Series}, 1970.

\bibitem[Rockafellar(2007)]{Rockafellar2007CoherentAT}
Terry Rockafellar.
\newblock Coherent approaches to risk in optimization under uncertainty.
\newblock 2007.

\bibitem[Rothblum \& Yona(2021)Rothblum and Yona]{rothblum2021multi}
Guy~N Rothblum and Gal Yona.
\newblock Multi-group agnostic pac learnability.
\newblock In \emph{International Conference on Machine Learning}, 2021.

\bibitem[Sagawa* et~al.(2020)Sagawa*, Koh*, Hashimoto, and Liang]{Sagawa*2020Distributionally}
Shiori Sagawa*, Pang~Wei Koh*, Tatsunori~B. Hashimoto, and Percy Liang.
\newblock Distributionally robust neural networks.
\newblock In \emph{International Conference on Learning Representations}, 2020.

\bibitem[Serraino \& Uryasev(2013)Serraino and Uryasev]{Serraino2013}
Gaia Serraino and Stanislav Uryasev.
\newblock \emph{Conditional Value-at-Risk (CVaR)}.
\newblock Springer US, 2013.

\bibitem[Vaicenavicius et~al.(2019)Vaicenavicius, Widmann, Andersson, Lindsten, Roll, and Sch\"{o}n]{pmlr-v89-vaicenavicius19a}
Juozas Vaicenavicius, David Widmann, Carl Andersson, Fredrik Lindsten, Jacob Roll, and Thomas Sch\"{o}n.
\newblock Evaluating model calibration in classification.
\newblock In \emph{International Conference on Artificial Intelligence and Statistics}, 2019.

\bibitem[Valiant(1984)]{valiant1984theory}
Leslie~G Valiant.
\newblock A theory of the learnable.
\newblock \emph{Communications of the ACM}, 1984.

\bibitem[Villani(2009)]{Villani2009}
C{\'e}dric Villani.
\newblock \emph{The Wasserstein distances}.
\newblock Springer Berlin Heidelberg, Berlin, Heidelberg, 2009.

\bibitem[Zhang et~al.(2024)Zhang, Zhan, Chen, Du, and Lee]{zhang2024optimal}
Zihan Zhang, Wenhao Zhan, Yuxin Chen, Simon~S Du, and Jason~D Lee.
\newblock Optimal multi-distribution learning.
\newblock In \emph{Conference on Learning Theory}, 2024.

\bibitem[Zhao et~al.(2021)Zhao, Kim, Sahoo, Ma, and Ermon]{zhao2021calibrating}
Shengjia Zhao, Michael Kim, Roshni Sahoo, Tengyu Ma, and Stefano Ermon.
\newblock Calibrating predictions to decisions: A novel approach to multi-class calibration.
\newblock \emph{Advances in Neural Information Processing Systems}, 2021.

\end{thebibliography}
\bibliographystyle{tmlr}

\appendix
\section{Calibration}\label{sec:appendix-results}
In this section, we give general background on calibration and loss minimization. We continue the notation from the main text. 
\subsection{Loss Minimization, Calibration, and Grouping}\label{sec:loss-minimisation-calibration-and-grouping}

Ideally, the goal in learning is to estimate the Bayes predictor $h^{*}\left(\vx\right) = P\left(\ry \ \vert \ \rvx = \vx\right) \in \Delta^{|\mathcal{Y}|}$ where the $i^{\text{th}}$ coordinate in this $|\mathcal{Y}|$-dimensional vector gives denotes $P\left(\ry = y_{i} \ \vert \ \rvx=\vx\right)$. However, this is generally difficult due to statistical, computational, and philosophical reasons, especially when $\vx$ is high-dimensional. In the absence of $h^{*}\left(\vx\right)$, one reasons probabilistically about $\vx$ via a group membership for $\vx$ in some group $G_{p}$ defined by some property $p$. Thus, instead of defining an \textit{individual} (conditional) probability for $\vx$, one considers a collection of inputs $\vx$ that are \textit{indistinguishable} according to some pre-specified property $p$, as $G_{p}: \{\vx \in \mathcal{X} \ \vert \ p\left(\vx\right)=1\}$. Having defined the grouping $G_{p}$ (or the reference class), the goal is to define (conditional) probability as $P\left(\ry \ \vert \ \vx \in G_{p}\right)$. The indistinguishable term underscores the inherent independence notion: for $\vx_1 \neq \vx_2$, conditional on the group membership $G_{p}$, i.e. $\vx_1, \vx_2 \in G_{p}$, the exact difference between $\vx_1$ and $\vx_2$ is not important, they are \textit{indistinguishable} according to the property $p$. There are a multitude of ways to define a grouping operator \citep{holtgen2023richness}, and it borrows with it the same issues as the classical reference class problem \citep{Hajek2007-HJETRC}. However, a major focus has been the default grouping induced by risk minimization which is defined by the prediction itself. Given the prediction $\nu$, $G_{\nu} = \{\vx \in \mathcal{X} \ | \ h\left(\vx\right) = \nu\}$. Then calibration says that $P\left(\ry \ \vert \ \ \vx \in G_{\nu}\right) = \nu$. The calibration-refinement decomposition for a proper scoring loss $\ell$ then says that loss-minimization leads to picking a predictor $h$ that is calibrated (as per this grouping definition) and well-refined. 

\subsection{Calibration and Confidence Thresholding Rule}\label{sec:calibration-thresholding-rule}
As stated in Proposition \ref{prop:calibration-and-decision-making-regular-calibration}, calibrated predictors can be reliably used to define a decision-rule based on the expected cost minimization rule. In this section, we give example that in the binary cases, when the action space $\mathcal{A}$ is binary, then the decision rule becomes thresholding the predictor's confidence based on a threshold defined the cost function $c: \mathcal{A} \times \mathcal{Y} \rightarrow \mathbb{R}_{+}$. This is practically useful in risk prediction scenarios, for example. Given a medical facility using some disease risk prediction system $h$ to help order some expensive test for further diagnosis, the action space $\mathcal{A}$ in this case is $\{\text{order test, not order test}\}$, and the cost function signal expenses in conducting the test which the healthcare facility wants to minimize. In such a setting, a cost function can be define by constants $c_{\text{TP}}, c_{\text{TN}}, c_{\text{FP}}, c_{\text{FN}}$ where $\text{TP}, \text{TN}, \text{FP}, \text{FN}$ denotes true positive, true negative, false positive, false negative, in order. Given a predictive system $h\left(\vx\right)$, the expected cost for action $a_1$ (e.g., order test) is $P\left(\ry = 1 \ \vert h\left(\vx\right) = \nu\right)c_{\text{TP}} + P\left(\ry = 0 \ 
 \vert \ h\left(\vx\right) = \nu \right) c_{\text{FP}}$, which due to calibration is equal to $\nu \cdot c_{\text{TP}} + (1-\nu)\cdot c_{\text{FP}}$. Similarly, the expected cost for action $a_2$ (e.g. not order test) is $\nu \cdot c_{\text{FN}} + (1-\nu)\cdot c_{\text{TN}}$. If one wants to use the predictor to recommend further tests, then $\nu \cdot c_{\text{TP}} + (1-\nu)\cdot c_{\text{FP}} > \nu \cdot c_{\text{FN}} + (1-\nu)\cdot c_{\text{TN}}$, which leads to $\frac{\nu}{1-\nu} > \frac{c_{\text{TN}} - c_{\text{FP}}}{c_{\text{TP}} - c_{\text{FN}}}$. Thus, one can use the considered cost to define a threshold on the predicted confidence to recommend the action. In general, we have the following definition:

 \begin{definition}
 Given prediction space $\mathcal{Y} = \{0,1\}$, the action space $\mathcal{A} = \{a_1, a_2\}$, and the cost structure defined by constants $c_{\text{TP}}, c_{\text{TN}}, c_{\text{FP}}, c_{\text{FN}}$, for the calibrated predictor $h : \vx \mapsto [0,1]$, an optimal action amounts to thresholding as 
 $$a_1 \ \ \text{if } \  \mathbb{I}\{f\left(h\left(\vx\right)\right) > g\left(c_{\text{TP}}, c_{\text{TN}}, c_{\text{FP}}, c_{\text{FN}}\right)\}, \ \ a_2 \ \ \text{otherwise},
 $$
 where $f$ and $g$ are some functions. 
\end{definition}
Proposition \ref{prop:calibration-and-decision-making-regular-calibration} then says that the above decision-rule is optimal on-average across all the predictions. The above example also asserts the importance of calibration as it allows one to incorporate cost structure in a post-hoc manner, as opposed to the alternative way of incorporating it during learning via the framework of cost-sensitive learning \citep{10.5555/1642194.1642224}.

\section{Decision-Theoretic Property of Calibration}\label{sec:decision-theoretic-calibration-proofs-section}

\subsection{Proof of Proposition \ref{prop:calibration-and-decision-making-regular-calibration}}\label{sec:calibration-and-decision-making}
Given the space $\mathcal{X} \times \mathcal{Y}$ and distribution $P$ on it, a forecaster $h(\rvx=\vx)$, some cost function $c$, and an action space $\mathcal{A}$. Given $\rvx=\vx$, the decision maker has to take some action, hence would use the forecast $h\left(\rvx=\vx\right)$ to make a decision. The forecast is over $\mathcal{Y}$---the prediction space which is usually not equal to $\mathcal{A}$. The decision-maker then takes the action $a^{*}\left(\vx\right) = \arg\min_{a \in \mathcal{A}}\mathbb{E}_{\ry \sim h\left(\vx\right)}c\left(a, \ry\right)$. Call this decision rule $\delta: h\left(\rvx\right) \rightarrow \mathcal{A}$: it takes some forecast and returns the action $\delta\left(h\left(\rvx\right)\right) \mapsto a^{*}$. Now, the quality of the said decision-rule $\delta$ needs to be evaluated. $\delta$ takes a random forecast, and the cost is evaluated with respect to the random $\ry$: the cost incurred for some random $h\left(\rvx\right)$ and $\ry$ is $c\left(\delta\left(h\left(\rvx\right)\right), \ry\right)$. To fully evaluate it: $P\left[c\left(\delta\left(h\left(\rvx\right)\right), \ry\right)\right]$ where the expectation is taken with respect to everything that is random. From the definition of (full) calibration: $\mathbb{E}\left[\ry \ \vert \ h(\rvx) = \hat{p}\right] = \hat{p}$, the above evaluation becomes:

\begin{align*}
\begin{split}
    \mathbb{E}_{h(\vx) \sim h\left(\rvx\right)}\mathbb{E}_{\ry \sim \ry \vert h\left(\vx\right)}\left[c(\delta(h\left(\vx\right), \ry)\right] &= \mathbb{E}_{h\left(\vx\right) \sim h\left(\rvx\right)}\left[\sum_{y_i \in \mathcal{Y}}\underbrace{P(\ry=y_i \ \vert  \ h(\vx))}_{\text{$=h(\vx)_{i}$: calibration}}[c(\delta(h(\vx)), y_i)]\right] \\ &= \mathbb{E}_{h\left(\vx\right) \sim h\left(\rvx\right)}\left[\sum_{y_i \in \mathcal{Y}}h(\vx)_{i}[c(\delta(h(\vx)), y_i)]\right] \\ &= \mathbb{E}_{h(\vx) \sim h\left(\rvx\right)}\mathbb{E}_{\ry \sim h\left(\vx\right)}[c(\delta(h(\vx)), \ry)]
\end{split}
\end{align*}
Now based on how $\delta$ is defined: $\delta$ takes a prediction $h(\vx)$ and returns an action $a^{*}$ that takes the minimal value $\mathbb{E}_{\ry \sim h(\vx)}c(a, \ry)$, so for any other decision-rule, the inner expectation is at-least the value of this inner expectation. Since this holds for any $h(\vx) \sim h(\rvx)$, the monotonicity of expectation says that the decision rule $\delta$ is optimal. In words, take a prediction, and take optimal action according to it, and if you average out across all the predictions, you're being optimal for any utility function. The crucial thing is that it is on average across all the predictions---the decision rule $\delta$ is optimal on average across the predictions. Or assuming the randomness in $h(\rvx) \times \ry$ is derived from $P$ over $\rvx \times \ry$, then on average across the whole population.

\subsection{Proof of Proposition \ref{prop:multi-distribution-learning-calibration-and-decision-theoretic-property}}\label{sec:proof-MDL-calibration-and-decision-theoretic-property}

\begin{proposition}
    \textit{(Calibration and decision-making)}. Given $\mathcal{Q}$ and a predictor $h$ calibrated with respect to the $Q^{*} \in \mathcal{Q}$ with the maximum generalized entropy for a loss function $\ell$, a decision-rule $\delta : h\left(\rvx\right) \mapsto \arg\min_{a \in \mathcal{A}}\mathbb{E}_{\ry \sim h\left(\rvx\right)}\left[c\left(a, \ry\right)\right]$ with the action space $\mathcal{A}$ and a cost function $c: \mathcal{A} \times \mathcal{Y} \rightarrow \mathbb{R}_{+}$ is optimal in a worst-case sense over $\mathcal{Q}$ as long as the cost function $c$ is consistent with the loss function $\ell$.
\end{proposition}
The argument builds upon the proof of proposition \ref{prop:calibration-and-decision-making-regular-calibration} in Section \ref{sec:calibration-and-decision-making} above.

In multi-distribution case, one considers a set   $\mathcal{Q}$ of distributions over $\mathcal{X} \times \mathcal{Y}$. Given a loss function $\ell$, it can be shown that multi-distribution learning is equivalent to learning with respect to the distribution $Q^{*} \in \mathcal{Q}$ that achieves the value $\sup_{Q \in \mathcal{Q}}\inf_{h \in \mathcal{H}}\mathbb{E}_{Q}\left[\ell\left(h\left(\rvx\right), \ry\right)\right]$---this value is termed as the maximum generalized entropy with respect to $\ell$. Or multi-distribution learning can be seen as a single-distribution learning with respect to $Q^{*}$, i.e. $\inf_{h \in \mathcal{H}}\mathbb{E}_{Q^{*}}\left[\ell\left(h\left(\rvx\right), \ry\right)\right]$. Assume the resulting $h$ is such that it is calibrated with respect to $Q^{*}$. For a decision maker with some cost function $c$, and the action space $\mathcal{A}$, the only available information to take the decision is the forecast $h\left(\vx\right)$ that is calibrated with respect to $Q^{*}$. This decision maker will then employ the same decision rule $\delta$: $\delta\left(h\left(\vx\right)\right) \mapsto \arg\min_{a \in \mathcal{A}}\mathbb{E}_{\ry \sim h\left(\vx\right)}c\left(a, \ry\right)$. Following from the single-distribution case, the decision rule $\delta$ is optimal in averages for $Q^{*}$, i.e. for any other decision rule $\delta^{'}$: $\mathbb{E}_{Q^{*}}\left[c\left(\delta\left(h\left(\rvx\right)\right), \ry\right)\right] \leq \mathbb{E}_{Q^{*}}\left[c\left(\delta^{'}\left(h\left(\rvx\right)\right), \ry\right)\right], \ \forall \delta^{'}$. Or $\delta = \arg\min_{\delta}\mathbb{E}_{Q^{*}}\left[c\left(\delta(h(\rvx)),\ry\right)\right]$. If we further assume that $Q^{*}$ also maximizes the generalized entropy with respect to the cost function $c$, under this assumption it can be said that $\delta$ achieves the value $\sup_{Q \in \mathcal{Q}}\min_{\delta}\mathbb{E}_{Q}\left[c\left(\delta(h(\rvx)),\ry\right)\right]$---a worst-case decision optimality against $\mathcal{Q}$.

\end{document}